\documentclass[smallextended]{svjour3}       

\usepackage{subfloat}
\usepackage{lipsum}
\usepackage{amsfonts}
\usepackage{graphicx}
\usepackage{epstopdf}
\usepackage{amsmath}

\usepackage{color}
\newcommand{\changes}{\color{black}}
\newcommand{\changess}{\color{black}}

  \smartqed  
  \usepackage{graphicx}
\begin{document}
  
  \title{ 
	Vanilla Feedforward Neural Networks as a Discretization of Dynamical {\changes Systems}
  }

  
  \author{Duan Yifei         \and
          Li Li'ang \and Ji Guanghua \and Cai Yongqiang 
  }
  
  
  \institute{
   Duan Yifei\\
  \email{202121130090@mail.bnu.edu.cn}\\           
  Li Li'ang\\
    \email{202121130091@mail.bnu.edu.cn}\\  
    Ji Guanghua\\
    \email{ghji@bnu.edu.cn}\\
    Cai Yongqiang (corresponding author) \\
    \email{caiyq.math@foxmail.com}\\
    School of Mathematical Sciences, Laboratory of Mathematics
and Complex Systems, MOE, Beijing Normal University
}

  \date{Received: date / Accepted: date}

  \maketitle
  
  \begin{abstract}
    Deep learning has made significant progress in the fields of data science and natural science. Some studies have linked deep neural networks to dynamical systems, but the network structure is restricted to a residual network. It is known that residual networks can be regarded as a numerical discretization of dynamical systems. In this paper, we consider the traditional network structure and prove that vanilla feedforward networks can also be used for the numerical discretization of dynamical systems, where the width of the network is equal to the dimensions of the input and output. Our proof is based on the properties of the leaky-ReLU function and the numerical technique of the splitting method for solving differential equations. Our results could provide a new perspective for understanding the approximation properties of feedforward neural networks.
  \end{abstract}

  \section{Introduction}\label{Introduction}

  Based on the universal approximation property of neural networks, deep learning plays an important role in artificial intelligence, data science and natural science, including applications on solving differential equations \cite{Han2018Solving,Yu2018Deep,Sirignano2018DGM,Raissi2019Physicsinformed}. Various network structures, such as feedforward neural networks (FNNs), convolutional neural networks (CNNs) \cite{krizhevsky2017imagenet}, recurrent neural networks (RNNs) \cite{lecun1998gradient}, and graph neural networks (GNNs) \cite{Wu2020Comprehensive}, have been developed to address different data structures. In practice, the networks, such as the well-known residual network (ResNet) \cite{he2016deep,he2016identity}, include skip connections, which is standard for almost all the best models, such as the remarkable AlphaFold network for protein structure prediction \cite{jumper2021highly}.
  Motivated by the combined ResNet and forward Euler method for discretizing dynamical systems \cite{Weinan2017A}, novel networks such as neural ordinary differential equations (ODE) \cite{Chen2018Neural}, PolyNet\cite{zhang2017polynet}, FractalNet\cite{larsson2017fractalnet}
  and RevNet \cite{lu2018beyond},
  have been proposed, and they have extended the structure library of neural networks.
  
  For the dynamical system $\dot{x}(t) = v(x(t),t)$ with initial value $x(0)=x_0 \in \mathbb{R}^d$, the traditional forward Euler discretization approach with time step $\Delta t$ is
  \begin{align}\label{eq:forward_Euler}
      x_{k+1} = x_k + \Delta t v(x_k, t_k)
      =: x_k + F(x_k; \Theta_k),
  \end{align}
  where $x_k$ approximates the state $x(k\Delta t)$ provided that the time step $\Delta t$ is sufficiently small. ResNet can be regarded as an approximation of (\ref{eq:forward_Euler}), where the right-hand term $F(x_k; \Theta_k)$ is approximated by neural networks \cite{Weinan2017A,Sander2022Residual}, called residual blocks with parameter $\Theta_k$. For example, taking each residual block as a single-hidden-layer feedforward network with tanh activation, the following ResNet can be obtained:
  \cite{Weinan2017A,haber2017stable,lu2018beyond}
  \begin{align}
      x_{k+1} = x_k + S_k \tanh (W_k x_k + b_k),
  \end{align}
  where $\Theta_k = (S_k,W_k,b_k)\in \mathbb{R}^{d \times N}\times \mathbb{R}^{N\times d} \times \mathbb{R}^N$ denotes the control parameters and
  $N$ is the number of hidden neurons. Theoretically, the flow map of a dynamical system has the property of homeomorphism. However, maintaining this property in a ResNet model is not easy in practice since the Lipschitz coefficient of each residual block needs to be constrained \cite{Behrmann2019Invertible}. To overcome this drawback, \cite{Dinh2014Nice,Dinh2016Density,teshima2020coupling} improved the network model by introducing NICE \cite{Dinh2014Nice} and NVP \cite{Dinh2016Density}, which are reversible and naturally maintain the property of homeomorphism.
  
  The vanilla feedforward network, which has the simplest structure, can maintain the property of homeomorphism, but this network does not receive the attention it deserves. In fact, for a feedforward network with a width equal to the input (and the output) dimension $d$, if the activation function is strictly monotonic and continuous (such as the leaky-ReLU activation) and the weight matrix of each layer is nonsingular, then the map from the input to the output is homeomorphic. In this paper, we focus on this configuration and the natural question below:
  \begin{quotation}
      \emph{Can we use feedforward neural networks with width $d$ to approximate the flow map of a dynamical system in $d$ dimensions?}
  \end{quotation}
  
  \textbf{Contribution.} This paper constructs leaky-ReLU {\changes neural networks} and answers the above question in the affirmative. Our construction is based on a proper splitting method to solve ODEs and an observation that the composition of leaky-ReLU activation and linear functions can approximate monotonic functions. The results imply that leaky-ReLU neural networks with width $d$ (the dimension of the input space) are a universal approximator for the flow map of dynamical systems.
  The significant point is that we do not need to increase the hidden dimension to achieve a universal approximation for flow maps.
  This indicates that residual-type networks, such as ResNet, are not unique structures corresponding to dynamical systems.
  In addition, an FNN has the simplest structure that preserves homeomorphic properties.
  These findings could deepen our understanding of the property of vanilla FNNs.
  
  \textbf{Related works.}
  In this paper, the discretization of dynamical systems is expressed by universally approximating their flow maps. Therefore, our results are related to the universal approximation property (UAP) for a broader function class. Studying the UAP of neural networks is fundamental for deep learning and has a long history. It is well known that FNNs are universal approximators for continuous functions \cite{Cybenkot1989Approximation,Hornik1989Multilayer,Leshno1993Multilayer}. As a consequence, the universal approximating flow maps are obvious. However, the traditional FNN UAP is applied to wide networks that do not preserve the homeomorphism property of flow maps.
 {\changes
 To maintain this property, NICE model \cite{Dinh2014Nice}, NVP model \cite{Dinh2016Density} and various invertible architectures \cite{Chen2018Neural,Chang2018Reversible,lu2018beyond,Behrmann2019Invertible,teshima2020coupling} are proposed. In addition, the UAP of ResNet and flow maps of ODEs are also investigated \cite{Lin2018Resnet,Li2022Deep,Tabuada2023Universal}. Different from these models, our configuration focuses on FNNs with a narrow width that is equal to the input and output dimensions to potentially preserve the homeomorphism property.
 }
  
 {\changes
 Recently, determining the minimal width of FNNs, \emph{i.e.} the minimum number of neurons per layer, for the UAP has attracted much attention \cite{Lu2017Expressive,Hanin2018Approximating,Johnson2019Deep,Kidger2020Universal,Park2021Minimum,Beise2020Expressiveness}. 
{\changess
 The goal of these studies is to provide upper and lower bounds on the minimum width of the networks.
} 
Generally, the lower bounds are obtained by providing counter-examples that can not be well approximated and the upper bounds are given by specific constructions. The construction in the above works can be summarized as some encoding schemes where additional dimensions (or neurons) are needed \cite{Park2021Minimum}. The construction in this paper is significantly different as we do not introduce any additional dimensions. 
  }

For continuous scalar functions on a compact domain in $\mathbb{R}^d$ and networks with monotone and continuous activations (such as the well-known ReLU and leaky-ReLU functions), the minimal width of the uniform UAP is $w_{\min} \ge d+1$ \cite{Johnson2019Deep}. This indicates that $d$ is a corner width, which prevents the UAP. The results can be understood from the property of level sets of a target function \cite{Hanin2018Approximating,Johnson2019Deep}, which is heavily related to topology theory. Our results go a step further and examine the approximation power of these networks, especially the leaky-ReLU {\changes neural networks}.

  \textbf{Outline.} We state the main result for leaky-ReLU networks in Section \ref{sec:main}, which includes ideas and key lemmas to derive the main result. Section \ref{sec:proof} provides a formal proof of the theorems, where the subsections correspond to the ideas listed in Section \ref{sec:main}. {\changes In Section \ref{sec:discussion} we discuss networks with more general activation functions. Finally, a brief conclusion is made in Section \ref{sec:conclusion}.}
  
  \section{\changess{Notations and main results}}
  \label{sec:main}
  
  \subsection{\changess{Neural networks}}
  
  An FNN consists of an input layer, hidden layers, and an output layer, which is the simplest neural network structure. Given an input $x  \in \mathbb{R}^{d}$ and by defining the output $h^{[l]} \in \mathbb{R}^{n_l}$ of the $l$-th layer (also the input of the $(l+1)$-th layer), we can define the following neural network:
  \begin{align}
      h^{[0]} & = x  \in \mathbb{R}^{n_0}, n_0 = d,\\
      z^{[k]} & = W^{[k]} h^{[k-1]} + b^{[k]} \in \mathbb{R}^{n_k}, \\
      h^{[k]} &= \sigma(z^{[k]}) \in \mathbb{R}^{n_k}, k=1,2,\dots ,L.
  \end{align}
  where $W^{[k]}\in \mathbb{R}^{n_{k-1} \times n_k}$ and $b^{[k]}\in \mathbb{R}^{n_k}$ are defined as the weights and biases of the $k$-th layer, and $\sigma$ is the activation function that operates on the argument in an elementwise manner. We use the preactivated state $z^{[L]}$ as the final output of the network and consider it as a function of $x$ denoted by $f_L(x) := z^{[L]}$. By this notation, we obtain the following recursive relations:
  \begin{align}
      f_{0}(x) &= W^{[0]} x+b^{[0]}, \\
      f_{k}(x) &= W^{[k]} \sigma \left(f_{k-1}(x)\right) + b^{[k]}, k=1,2,\cdots,L.
  \end{align}
  
  In this paper, we consider the case of equal-width layers, including the input and output layers; \emph{i.e.} $n_0=n_1=...=n_L=d$. In addition, we focus on leaky ReLU activation (first introduced by \cite{maas2013rectifier}), which is $\sigma_{\alpha}(x), \alpha \in (0,1) \cup (1,+\infty),$
  \begin{align}
      \sigma(x) \equiv \sigma_{\alpha}(x) = 
      \begin{cases}
          x, &x > 0, \\
          \alpha x , & x \leq 0.
      \end{cases}
  \end{align}
  With this setup, the networks are characterized by the depth $L$ and the parameters $W$ and $b$. We denote the set of all $f_L(x)$ with such networks by $\mathcal{N}_d(L)$ and denote
  \begin{align}
      \mathcal{N}_d = \bigcup_{L=0}^\infty \mathcal{N}_d(L).
  \end{align}
  
  \subsection{\changess{Dynamical systems}}
  
  Another object in this paper is the following ODE system in dimension $d$:
  \begin{align}\label{eq:ODE_general}
      \left\{
      \begin{aligned}
      &\dot{x}(t) = v(x(t),t), t\in(0,\tau),\\
      &x(0)=x_0 \in \mathbb{R}^d.
      \end{aligned}
      \right .
  \end{align}
  The flow map from $x_0$ to $x(\tau)$ is denoted by $\phi^\tau(x_0)$. Although the flow map $\phi^\tau$ is defined on the whole space, in this paper, we mainly consider its constraint in a compact domain $\Omega$ in $\mathbb{R}^d$. If the field function $v(x(t),t)$ is expressed by a neural network then we call the ODE in (\ref{eq:ODE_general}) a neural ODE. To avoid the discussion on the existence and regularity of the flow map, we introduce the following mild assumption on the field function in (\ref{eq:ODE_general}).
  \begin{assumption}\label{th:assumption}
  The field function $v(x,t)$ is bounded, Lipschitz continuous with $x$ in $\mathbb{R}^d$ and piecewise smooth w.r.t. $t\in(0,\tau)$ (in each piece, $v(x,t)$ is continuous w.r.t $(x,t)$).
  \end{assumption}
  
  Note that we introduce the boundedness assumption on $v(x,t)$ just to make the proofs in this paper more concise. If $v(x,t)$ is unbounded on $\mathbb{R}^d$, we can modify it to be bounded without affecting the flow map on the compact domain $\Omega$ that we considered.
  
  \subsection{Feedforward neural \changess{networks} as a discretization of dynamical systems}
  
  Here, we formally state our main theorem.
  
  \begin{theorem}\label{th:main}
  
  Let $\phi^\tau$ be the flow map of the dynamical system in (\ref{eq:ODE_general}) with assumption \ref{th:assumption}, and $x$ is in a compact domain $\Omega \subset \mathbb{R}^d$.
  Then, for any $\varepsilon>0$ and $\tau>0$, there exists a leaky-ReLU network $f_L(x)$ with width $d$ and depth $L$ such that
  $\|f_L(x)-\phi^\tau(x)\| \le \varepsilon$
  for all $x \in \Omega$.
  \end{theorem}
  
  A detailed proof is provided in the next section. Here, we provide a sketch of the proof.
  Let us begin with the one-dimensional case, $d=1$. In this case, the flow map $\phi^\tau(x)$ is a monotonic and continuous function, and the network with depth $L$ is
  \begin{align}
      f_{L}(x)=w_{L} \sigma_\alpha (
          w_{L-1} \sigma_\alpha ( ... (w_0 x + b_0) ...) + b_{L-1}
          )+b_{L},
          w_{i}, b_{i} \in \mathbb{R}.
  \end{align}
  It is easy to verify that $f_L(x)$ is monotonic and continuous. With a detailed characterization of the leaky-ReLU networks in $\mathcal{N}_1$, we can prove the inverse result that any monotonic continuous function on a bounded interval $I$ (such as $I=[0,1]$) can be approximated by functions in $\mathcal{N}_d$. This is the lemma below, which is important for our construction of high-dimensional cases.
  \begin{lemma}\label{th:main_1d}
  For any monotonic function {\changes $u \in C(I,\mathbb{R})$} and $\varepsilon>0$, there is a leaky-ReLU neural network $f_L(x)$ with depth $L$ and a width of one such that $\|f_L(x) - u(x)\| \le \varepsilon$ for all $x\in I$.
  \end{lemma}

  {\changes
  \begin{remark} \label{remark:L_bound_1d}
        If $u \in C^2(I,\mathbb{R})$ is a smooth and strictly monotonic function, then the required depth $L$ can be bounded by $L \le \tfrac{V(u)}{\varepsilon} + \tfrac{V(\ln u')}{|\ln \alpha|}$ where $V(u)$ and $V(\ln u')$ means the total variation of $u$ and $\ln u'$, respectively. 
  \end{remark}
  }
  
  For high dimension cases, that is, when $d\ge 2$, our construction requires three steps, where the first two steps are almost standard. The ideas are also illustrated in Figure~\ref{fig:main_figure}. 
  
  \textbf{Step 1. Approximate the dynamical system by neural ODEs.}
  We first use neural ODEs to approximate the original dynamical system (\ref{eq:ODE_general}). This procedure is standard, as we can uniformly approximate the field $v(x,t)$ by single-hidden layer networks. We can control the approximation of $v(x,t)$ such that the approximation error for $\phi^\tau(x)$ is sufficiently small. Particularly, we approximate $v(x,t)$ by the following tanh network $\tilde v(x,t)$ with $N$ hidden neurons:
  \begin{align} \label{eq:specified_v}
      \tilde v(x,t) = 
      \sum\limits_{i=1}^N a_i(t)\tanh(w_i(t) \cdot x+b_i(t))
      \equiv
      A(t)\tanh(W(t)x(t)+b(t)),
  \end{align}
  where $(A,W,b)\in \mathbb{R}^{d \times N} \times \mathbb{R}^{N\times d} \times \mathbb{R}^N$ denotes the piecewise smooth functions of $t$ and $a_i$, $w_i$ and $b_i$ are the $i$-th row of $A,W$ and $b$, respectively.
  
  \textbf{Step 2. Solve the neural ODE by a proper splitting method.}
  Note that our aim is now to use a leaky-ReLU FNN to approximate the flow map of the neural ODE:
  \begin{align}\label{eq:ODE_neural}
      \dot x(t) = \tilde{v}(x(t),t).
  \end{align}
  A proper numerical scheme is desired to discretize the neural ODE (\ref{eq:ODE_neural}), which should benefit the construction in the next step. 
  {\changes 
  Unlike the ResNet, which corresponds to the forward Euler scheme solving the neural ODE (\ref{eq:ODE_neural}), we employ the splitting method before using the forward Euler scheme.
  } 
  Matrix $A$ in (\ref{eq:specified_v}) is split into $A(t) = \sum_{i=1}^N \sum_{j=1}^d a_{ij}(t) E_{ij},$
  where $a_{ij}$ is the $(i,j)$ element of $A$ and $E_{ij}$ is a matrix that is the same size as $A$ and has 1 as its only nonzero element at $(i,j)$. Then, the field $\tilde v$ can be split as a summation of $Nd$ functions,
  \begin{align}\label{eq:splitting_of_tilde_v}
      \tilde v(x,t) = \sum_{i=1}^N \sum_{j=1}^d a_{ij}(t) \vec{e}_j \tanh(w_i(t) \cdot x+b_i(t))
      \equiv
      \sum_{i=1}^N \sum_{j=1}^d \tilde v_{ij}(x,t) \vec{e}_j,
  \end{align}
  where $\vec{e}_j$ is the $j$-th axis unit vector and
  $$\tilde v_{ij}(x,t) = a_{ij}(t) \tanh(w_i(t) \cdot x+b_i(t))$$
  is a scalar function.
  For a given time step $\Delta t$, the $k$-th iteration is defined as
  \begin{align}\label{eq:main_iteration_Tk}
      x_{k+1} = T_k x_k = T_k^{(N,d)} \circ \dots \circ T_k^{(1,2)} \circ T_k^{(1,1)} x_k,
  \end{align}
  where the map $T_k^{(i,j)}: x \to y$ in each split step is
  \begin{align}\label{eq:map_T}
      \left\{
      \begin{aligned} 
      & y^{(l)} = x^{(l)} , l \neq j,  \\
      & y^{(j)} = x^{(j)} + \Delta t \tilde v_{ij}(x,t_k), t_k = k\Delta t.
      \end{aligned}
      \right.
  \end{align}
  Here, the superscript in $x^{(l)}$ indicates the $l$-th coordinate of $x$. Employing some standard numerical analysis techniques, we can prove that $x_k$ approximates $x(k \Delta t)$ arbitrarily, provided $\Delta t$ is sufficiently small and $k \Delta t \le \tau$.
  
  \textbf{Step 3. Approximate each split step by an FNN.}
  Our key observation is that all maps $T_k^{(i,j)}$ in (\ref{eq:map_T}) can be approximated by leaky-ReLU networks. If $\Omega_k$ is a compact set and $\Delta t$ in map $T_k^{(i,j)}$ is sufficiently small, then for any $\varepsilon>0$, there is a leaky-ReLU network $\mathcal{T}_k \in \mathcal{N}_d$ such that
  \begin{align}
      \|T_k(x) - \mathcal{T}_k(x)\| \le \varepsilon, \forall x \in \Omega_k.
  \end{align}
  The detailed construction is shown in Section 3.5. The basic observation is that $\tanh(\nu)$ and $\nu + \Delta t C \tanh(\nu)$ are strictly monotonic functions of $\nu \in \mathbb{R}$, provided that $C$ is a fixed constant and $\Delta t$ is sufficiently small. Then, they can be approximated by leaky-ReLU networks with a width of one according to the one-dimensional result. The specific structure of $T_k^{(i,j)}$ is designed so that we can approximate it by the composition of a linear transformation and (one-dimensional) leaky-ReLU activations.
  
  Combining the three steps above, we are now in a position to show that the FNN $f_L(x)$ can approximate the flow-map $\phi^\tau(x)$ of the dynamical system in (\ref{eq:ODE_general}) up to any accuracy, where $\tau = n \Delta t$, $n\in \mathbb{Z}^+$ and
  \begin{align}
      f_L = \mathcal{T}_n \circ \cdots \circ \mathcal{T}_2 \circ \mathcal{T}_1 \in \mathcal{N}_d.
  \end{align}
  In addition, the constructed $f_L(x)$ is naturally a homeomorphic map, which is a key property of flow maps.
  Furthermore, if we are interested in the whole trajectory of $x(t)$, discretized as $x(k \Delta t)$, instead of the final flow-map, we can extract a subsequence of the hidden state of $f_L(x)$, denoted by $f_{L_k}(x)$, such that $f_{L_k}(x)$ approximates $x(k \Delta t)$ well. From this sentence, we concluded that a vanilla FNN could be used to discretize dynamical systems.
  
  \begin{figure}[t]
      \centering
      \includegraphics[width=12cm]{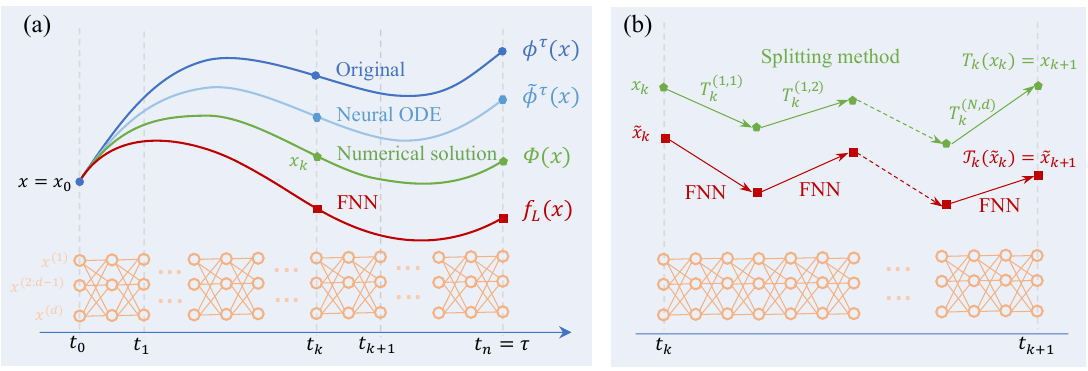}
      \caption{Approximating the flow map $\phi^\tau$ by an FNN $f_L$. (a) Illustration of the three steps in our proof. (b) Numerically solving the neural ODE by a splitting method, where each split step can be approximated by an FNN.}
      \label{fig:main_figure}
  \end{figure}

  {\changes
  \begin{remark}
        The mappings $T_k^{(i,j)}$ in (\ref{eq:map_T}) can be viewed as simple residual blocks, containing only one neuron and modifying only one coordinate, of a ResNet. Therefore, our results can be extended to residual networks, i.e., ResNets with only one neuron in each layer can approximate flow maps. This is a strong complement to an existing work \cite{Lin2018Resnet} which demonstrated that ResNets with one-neuron hidden layers can approximate scalar Lebesgue-integrable functions. Our results imply that for flow maps, the approximation can be achieved under the uniform norm that is stricter than the $L^1$-norm.
  \end{remark}
  }
  
  \section{Proof of the results}
  \label{sec:proof}
  
  \subsection{Preliminary properties of leaky-ReLU}
  \label{sec:preliminary_properties}
  
  Leaky-ReLU functions and their compositions are piecewise linear (PL) functions. To simplify the proofs in this paper, we introduce a notation for (scalar) PL functions.
  
  \begin{definition}[Piecewise linear (PL) functions]
  We denote a PL function with a finite number (denoted by $n$) of breakpoints by $f(x)=f(x;\Theta)$, $x\in \mathbb{R}$, and $\Theta=(B,S,V)$, in which $B (\in \mathbb{R}^n)$ represents the sequence of breakpoints, $S (\in \mathbb{R}^{n+1})$ represents the slope of each linear part of $f(x)$, and $V (\in \mathbb{R}^{n})$ represents the values of $f(x)$ at $x\in S$.
  \end{definition}
  We denote $B(f), S(f), V(f)$ as the parameter $B,S,V$ of $f(x)$.
  
  \begin{definition}[$\alpha$-power PL functions]
  We call a PL function $g(x)$ an $\alpha$-power PL function {\changess ($\alpha>0$) }with a constant $c \in \mathbb{R}$; if $S(g) \subset \{ c \alpha^k,k \in \mathbb{Z}\}$ for some constant $c$, \emph{i.e.,} the slopes of each piece are proportional to the power of $\alpha$.
  \end{definition}
  
  Note that leaky ReLU has some useful properties, which are listed in the following proposition.
  
  \begin{proposition} \label{prop:sigma_alpha_p}
  The leaky ReLU functions, $\sigma_\alpha$, have the following properties:
  \begin{itemize}
  \item[1)] Positive-homogeneous,
  $$\sigma_\alpha(a x) = a \sigma_\alpha(x), \forall a>0.$$
  \item[2)] Representation of identity,
  $$\frac{1}{\alpha} \sigma_\alpha(-\sigma_\alpha(-x)) = x.$$
  Furthermore, for $p\in\mathbb{Z}^+$,
  $$ \sigma _{\alpha^p} (x)=\sigma _{\alpha} \circ \cdots \circ \sigma _{\alpha} (x),$$
  $$\sigma _{\alpha^{-p}} (x) = -\alpha^{p}\sigma_{\alpha} \circ \cdots \circ \sigma_{\alpha} (-x),$$
  where $\sigma_{\alpha}$ appears $p$ times.
  \end{itemize}
  \end{proposition}
  \begin{proof} The properties are obvious by a direct evaluation.
  \end{proof}
  The second property implies that the leaky-ReLU $\sigma _{\alpha^p} (x)$ with parameter $\alpha^p, p \in \mathbb{Z},$ can be represented as a composition of $\sigma _{\alpha} (x)$; hence, $\sigma _{\alpha^p} (x) \in \mathcal{N}_1$.

  \subsection{Characterizing the leaky-ReLU networks with width one}
  
  Here, we focus on the one-dimensional case, $d=1$. The following two lemmas show that leaky-ReLU networks with widths of one are exactly $\alpha$-power PL functions.
  
  \begin{proposition}\label{proposition:nn_PL}
  Leaky-ReLU networks in $\mathcal{N}_{1}(L)$ are $\alpha$-power PL functions with at most $L+1$ pieces.
  \end{proposition}
  \begin{proof}
  Let us prove this by induction. Consider $g(x) \in \mathcal{N}_{1}(L)$ as an $\alpha$-power PL function with a constant $c$ and at most $L+1$ pieces; we prove that $ w \sigma(g(x)) +b $ is an $\alpha$-power PL function with constant $w c$ and at most $L+2$ pieces.
  
  The proof is obvious for the case of $c=0$, \emph{i.e.,} $g(x)$ is a constant function. Next, we only need to consider the nonconstant case ($c \neq 0$), which requires all weights in the leaky-ReLU networks to be nonzero. As a consequence, $g(x)$ is continuous and strictly monotone, and its range is the whole space $\mathbb{R}$. Without loss of generality, we assume that $g(x)$ is increasing, and the decreasing case can be proven in a similar way. According to the zero-point existence theorem, there is a unique $\xi \in \mathbb{R}$ such that $g(\xi)=0$, and
  \begin{align}
      \sigma(g(x))=
      \begin{cases}
          g(x), &x > \xi, \\
          \alpha g(x) , & x \leq \xi.
      \end{cases}
      \end{align}
  Therefore, $\sigma (g(x))$ is a PL function with breakpoints $B(\sigma(g)) = \{\xi\} \cup B(g)$. The number of breakpoints increases by at most one; \emph{i.e.,}, $\sigma (g(x))$ has at most $L+2$ pieces. In addition, the slopes $S(\sigma(g)) \subset S(g) \cup \alpha S(g)$ are also $\alpha$-power functions with constant $c$. These complete the proof since $ w \sigma(g(x)) +b$ is just a linear map of $\sigma(g(x))$, which only changes the constant $c$ to $w c$ in the slope set.
  \end{proof}
  
  \begin{lemma}\label{lemma:alpha_NN}
  Any $\alpha$-power PL function $g(x)$ can be represented as a leaky-ReLU network with a width of one; \emph{i.e.} there exist a positive integer $L$ and the corresponding parameters $w_i,b_i \in \mathbb{R}$ such that
  $$g(x) = f_L(x), \forall x \in \mathbb{R}.$$
  \end{lemma}
  \begin{proof}
  Without loss of generality, we may assume $g(x)$ is a monotonically increasing function with breakpoints $B(g) = \{x_1,...,x_k\}$, decreasing $\{x_i\}$ values, slopes $S(g) = \{1,\alpha^{p_1}, ..., \alpha^{p_k}\}$, and $k \in \mathbb{N}$. Here, we assume that the slope is 1 on the interval $[x_1,+\infty)$, the slope is $a^{p_i}$ on $[x_{i+1},x_i], i=1,2,...,k-1,$ and the slope is $a^{p_k}$ on $(-\infty,x_k]$. We represent $g(x)$ by leaky-ReLU networks piece by piece.
  
  As the slope of the piece on $[x_2,x_1]$ is $\alpha^{p_1}$, we can construct a composition function
  $$\widetilde{f}_1(x)= \sigma_{\alpha^{p_1}}(x-x_1)+g(x_1),$$
  which matches $g(x)$ on the interval $[x_{2},\infty)$. Performing the same procedure to construct $\tilde f_2,...., \tilde f_k$ by
  \begin{align*}
          \widetilde{f}_{i}(x)= 
          \sigma_{\alpha^{p_i-p_{i-1}}}(
              \widetilde{f}_{i-1}(x)-
              \widetilde{f}_{i-1}(x_i)
              )+g(x_{i}), i=2,...,k,
      \end{align*}
  we have that $\widetilde{f}_{k}(x)$ matches $g(x)$ on the whole space $\mathbb{R}$. 
  {\changes
  Define $L$ as the integer
  \begin{align}
        L = \sum_{i=1}^k |p_i - p_{i-1}|, \quad p_0 = 0,
  \end{align}
  then according to Proposition \ref{prop:sigma_alpha_p}, $f_L = \widetilde{f}_{k}$ is a leaky-ReLU network with depth $L$, which completes the proof.
  }
  \end{proof}
  
  \subsection{Approximation power of networks with widths of one}
  
  Here, we prove the main result for one-dimensional networks, $d=1$. We approximate a monotone function by PL functions, which are then approximated by leaky-ReLU networks (see Figure~\ref{fig:1ddemo}).
  
  \begin{lemma}\label{th:alpha_PL_monotonic}
  Let $u(x) \in C([-1,1])$ be a monotonic function; then, for any $\varepsilon >0$, there exists an $\alpha$-power PL function $g(x)$ such that
  \begin{align}
          |u(x)-g(x)| \le \varepsilon,
          \quad 
          \forall x\in [-1,1].
      \end{align}
  \end{lemma}
  \begin{proof}
  Without loss of generality, we assume $u(x)$ is monotonically increasing. Let $\tilde{u}(x):=u(x)+\frac{\varepsilon}{2}x$, which is strictly monotonically increasing. Let {\changes $\Delta h =\frac{\varepsilon}{2}$}; according to the skills of numerical analysis, there is a strictly increasing and continuous PL function $h (x)$ with breakpoints $\{x_1,\cdots,x_n\}$, {\changes $x_1=1, x_n=-1$} that satisfies (see Figure \ref{fig:1ddemo})
  
  \begin{align}
          h(x_i)=\tilde{u}(x_i),
          &\quad i = 1,...,n,\\
          \lvert h(x_{i+1})-h(x_{i})\rvert \le \Delta h,
          &\quad
          i = 1,...,n-1.
      \end{align}
  As a consequence, $h(x)$ approximates $\tilde u(x)$ well with uniform errors that are less than $\Delta h$ on the whole interval $[-1,1]$,
  \begin{align}\label{eq:monotonic_inequality}
          \lvert h(x)-\tilde{u}(x)\rvert \le h(x_{i})-\tilde{u}(x_{i+1})\le \Delta h,  
          \quad
          x\in[x_{i+1},x_{i}], i=1,...,n-1.
      \end{align}
  
  Next, we construct an $\alpha$-power PL function $g(x)$ to approximate $h(x)$. The idea is to fold each piece of $h(x)$ into two pieces with $\alpha$-power slopes (see Figure \ref{fig:1ddemo}(a)). 
  {\changes 
  Consider the $i$-th piece of $h(x)$ as
  \begin{align*}
          h(x) = s_i x + b_i, \ x \in [x_{i+1},x_{i}].
      \end{align*}
  For any fixed positive constant $C$, there exist two integers, $p_{i1},p_{i2}$, such that
  $C \alpha^{p_{i1}} \le s_i < C \alpha^{p_{i2}}$, where $p_{i2}$ can be set as $p_{i1} -1$ if $\alpha \in (0,1)$ and $p_{i1} +1$ if $\alpha >1$.
  Then, the following PL functions, $g^\pm(x)$, can approximate $h(x)$ well:
  \begin{align*}
          g^+(x) &= \left\{
              \begin{aligned} 
              C \alpha^{p_{i2}} (x - x_{i+1}) + h(x_{i+1}), &&& x \in [x_{i+1},\xi_i],\\
              C \alpha^{p_{i1}} (x - x_{i}) + h(x_{i}), &&& x\in [\xi_i,x_{i}],
              \end{aligned}
              \right.\\
          g^-(x) &= \left\{
              \begin{aligned} 
              C \alpha^{p_{i1}} (x - x_{i+1}) + h(x_{i+1}), &&& x \in [x_{i+1},\xi'_i],\\
              C \alpha^{p_{i2}} (x - x_{i}) + h(x_{i}), &&& x\in [\xi'_i,x_{i}],
              \end{aligned}
              \right.
      \end{align*}
  where $\xi$ and $\xi'$ are assigned such that $g^+$ and $g^-$ are continuous functions. Repeating the construction on each interval, choosing either $g^+$ or $g^-$ and concatenating them together, we can obtain an $\alpha$-power function $g$ with at most $2n-1$ breakpoints. (In fact, we can choose a proper assignation of $p_{i1}$ and $p_{i2}$ such that $g$ has at most $n$ pieces.)
  
  Finally, noticing that $g(x_i)=h(x_i)=\tilde u(x_i)$, $\tilde u$ and $g$ are monotonically increasing, and using the same argument in (\ref{eq:monotonic_inequality}), we have
  \begin{align*}
      |u(x)-g(x)| \le 
      |u(x) - \tilde{u}(x)| 
      + |\tilde{u}(x)-g(x)| 
      \le 
      \frac{\varepsilon}{2} + \Delta h
      \le \varepsilon, \quad \forall x \in [-1,1].
      \end{align*}
  }
  \end{proof}
  
  \begin{figure}[htp]
      \centering
      \includegraphics[width=12cm]{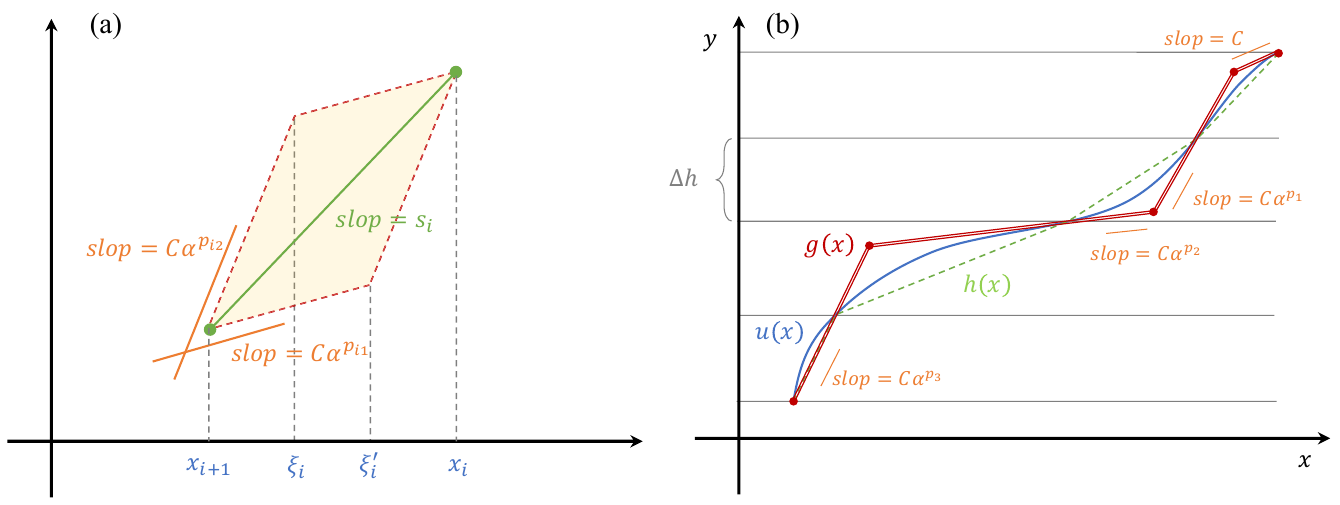}
  \caption{Approximating monotonic function $u(x)$ by $\alpha$-power PL function $g(x)$; (a) approximation of the single linear piece case; (b) approximation of $u$ by PL function $h$, which is approximated by $\alpha$-power PL function $g$.}
      \label{fig:1ddemo}
  \end{figure}
  
  With the lemmas and propositions above, we can draw our universal approximation theorem for a one-dimensional monotonic function as follows.
  
  \begin{proof}[Proof of Lemma \ref{th:main_1d}]
  For any interval $I=[a,b]$, we can scale it to $[-1,1]$ by a linear function. Then, according to Proposition \ref{th:alpha_PL_monotonic}, which shows that the $\alpha$-power PL function $g(x)$ can approximate $u(x)$, and Lemma \ref{lemma:alpha_NN}, which verifies the equivalence of the $\mathcal{N}_{1}$ and $\alpha$-power PL functions, we can prove that there exists $L\in \mathbb{Z}^+$ such that $|u(x)-f_L(x)|\le \varepsilon$ for all $x \in I$.
  \end{proof}

  {\changes
  
  To estimate the required depth $L$ to approximate a smooth and strictly monotonic function $u \in C^2(I,\mathbb{R})$, we only need to count the composition times in the constructed leaky-ReLU network $f_L$. To approximate $u$ with an error less than $\varepsilon$, a PL function $h$ with $k:= \lfloor V(u)/\varepsilon \rfloor  $ breakpoints is enough. 
{\changess
 Then we approximate $h$ by an $\alpha$-power PL function $g$, \emph{i.e.}, by a leaky-ReLU network $f_L$, constructed using the method described in the proof of Lemma~\ref{th:alpha_PL_monotonic}.
}
According to specific chosen powers $p_i$ of $g$, the depth $L$ satisfies
  \begin{align}
 L = \sum_{i=1}^k |p_i - p_{i-1}|
 \le
        \sum_{i=1}^{k} \Big( \Big|\frac{\ln s_{i-1}}{\ln \alpha} - \frac{\ln s_i}{\ln \alpha} \Big|+1 \Big).
  \end{align}
 Since the slope $s_i$ on the interval $[x_{i+1}, x_{i}]$ is set by $s_i = \frac{\varepsilon}{x_{i} - x_{i+1}}$, we have
  \begin{align}
        \sum_{i=1}^{k} |\ln s_{i-1} - \ln s_i| 
        =& 
        \sum_{i=1}^{k} |\ln \frac{\varepsilon}{x_{i-1}-x_{i}} - \ln \frac{\varepsilon}{x_{i}-x_{i+1}}| \nonumber\\
        =&
        \sum_{i=1}^{k} |\ln u'(\eta_{i-1}) - \ln u'(\eta_{i}) | 
        \le 
        V(\ln u'),
  \end{align}
 where the second equality employed the differential mean value theorem with mean values $\eta_i \in (x_{i+1}, x_{i})$ and the last inequality is from the definition of the total variation of the function $\ln u'$. 
 Finally, we have the estimation that
  \begin{align}
 L \le \tfrac{V(u)}{\varepsilon} + \tfrac{V(\ln u')}{|\ln \alpha|}.
  \end{align}
 It should be emphasized that when $\alpha$ tends to zero, the bound tends to ${V(u)}/{\varepsilon}$; {\changess Although the leaky-ReLU converges to ReLU locally, \emph{i.e.,} $\lim_{\alpha\to 0} \sigma_{\alpha}(x)=\text{ReLU}(x)$ for any fixed $x\in \mathbb{R}$, the limiting bound for $L$ does not hold for ReLU networks with width one, as their approximation power is poor (see Proposition \ref{th:relu_one}). In other words, Lemma \ref{th:main_1d} does not hold if ReLU activation is used instead of leaky-ReLU.}
  }
  
  \subsection{Approximate dynamical systems by neural ODEs} Here, we show that the flow map of the original ODE in (\ref{eq:ODE_general}) can be approximated by that of neural ODEs up to any accuracy. This is based on the error estimation between two ODE systems.
  
  \begin{lemma}\label{th:ODE_error_estimation}
  Consider two ODE systems
  $$\dot{x}(t) = f_i(x(t),t), t\in(0,\tau), i=1,2,$$
  where $f_i(x,t)$ are continuous w.r.t $x\in\mathbb{R}^d$ and piecewise continuous w.r.t $t \in [0,\tau]$. In addition, we assume $f_1(x,t)$ is bounded by $M$ and Lipschitz continuous with constant {\changes $\tilde L$}; \emph{i.e.,} $\|f_1(x,t)\| \le M$, $\|f_1(x,t)-f_1(x',t)\| \le \tilde L \|x-x'\|$. Then, for any compact domain $\Omega$ and $\varepsilon>0$, there is $\delta \in (0,1]$ such that the flow maps $\phi_1^\tau(x) ,\phi_2^\tau(x)$ satisfy
  \begin{align}
          \|\phi_1^\tau(x)- \phi_2^\tau(x)\| \le \varepsilon,
      \end{align}
  for all $x\in\Omega$ provided
  \begin{align}\label{eq:f_approx_condition}
          \|f_1(x,t)-f_2(x,t)\|< \delta, \quad\forall (x,t) \in \Omega_{\tau}\times[0,\tau],
      \end{align}
  where $\Omega_\tau$ is a compact domain defined as
  \begin{align}\label{eq:define_Omega_tau}
          \Omega_\tau = \{x + (M+1)\tau e^{\tilde L\tau} x': x \in \Omega, \|x'\| \le 1\}.
      \end{align}
  \end{lemma}
  \begin{proof}
  First, let us consider the case that inequality (\ref{eq:f_approx_condition}) is satisfied for all $(x,t) \in \mathbb{R}^d\times[0,\tau]$. Consider the error $e(x,t)$,
  $$e(x,t):=\phi_1^t(x) - \phi_2^t(x),$$
  and use the integral form of an ODE
  $$\phi_i^t(x) = \phi_i^0(x) + \int_0^t f_i(\phi_i^s(x),s) d s.$$
  Then, we have
  \begin{align*}
          \|e(x,t)\|
          &=\|\int_0^t 
          f_1(\phi_1^s(x),s) - f_2(\phi_2^s(x),s) ds
          \|\\
          &\le \int_0^t (
              \|f_1(\phi_1^s(x),s)-f_1(\phi_2^s(x),s)\|
              +
              \|f_1(\phi_2^s(x),s)-f_2(\phi_2^s(x),s)\|
              )ds\\
          &\le \int_0^t 
              (\tilde L \|e(x,s)\|+\delta)
              ds
          \le \delta t + \tilde L \int_0^t \|e(x,s)\| d s. 
      \end{align*}
  Employing the Gronwall inequation, we have $\|e(x,t)\| \le \delta t e^{\tilde L t}$. Then, letting $\delta = \min(1,\frac{\varepsilon}{\tau e^{\tilde L \tau}})$, we have $\|e(x,t)\| \le \varepsilon$ for all $x \in \mathbb{R}^d$ and $t \in [0,\tau]$.
  
  Now, we relax the inequality in (\ref{eq:f_approx_condition}) to $(x,t) \in \Omega_{\tau}\times[0,\tau]$. This relaxation only requires that $\phi_i^s(x) \in \Omega_\tau$ for all $x \in \Omega$ and $s \in [0,\tau]$. Since $f_1$ is bounded and Lipschitz, we have
  \begin{align*}
          \|\phi_1^s(x) -x\| \le M \tau e^{\tilde L\tau}.
      \end{align*}
  Accompanying $\|e(x,s)\| \le \delta s e^{\tilde L s} \le \tau e^{\tilde L \tau}$, we have $\phi_i^s(x) \in \Omega_\tau$ given by (\ref{eq:define_Omega_tau}). This completes the proof.
  \end{proof}
  
  \begin{lemma}\label{th:NODE_approx_DS}
  Let $\phi^\tau(x)$ be the flow map of (\ref{eq:ODE_general}) that satisfies Assumption \ref{th:assumption}, and let $\Omega$ be a compact domain. For any $\varepsilon>0$, there is a neural ODE, whose flow map is denoted by $\tilde{\phi}^\tau(x)$ with field function $\tilde{v}(x,t)$,
  \begin{align}\label{eq:tanh_NN}
          \tilde{v}(x,t) = \sum\limits_{i=1}^N a_i(t) \tanh(w_i(t) \cdot x+b_i(t)),
      \end{align}
  in which $N\in\mathbb{Z}^+$ is the number of hidden neurons and $a_i,w_i,b_i$ are piecewise smooth functions of $t$ such that $\|\phi^t(x) - \tilde{\phi}^t(x)\| \le \varepsilon$ for all $x \in \Omega$ and $t\in[0,\tau]$.
  \end{lemma}
  \begin{proof}
  According to the assumption on $v(x,t)$ and the universal approximation property of neural networks, for any $\delta>0$ there exists $N \in \mathbb{Z}^+ $ and tanh neural network (\ref{eq:tanh_NN}) such that $\|\tilde{v}(x,t)-v(x,t)\|<\delta$ for all $(x,t) \in \Omega_\tau \times [0,\tau].$ Then, using Lemma \ref{th:ODE_error_estimation}, we have $\|\phi^t-\tilde{\phi}^t\| \le \varepsilon$.
  \end{proof}
  
  \subsection{Solving neural ODEs with splitting methods}
  
  It is well known that an ODE system can be approximated by many numerical methods. Particularly, we use the splitting approach \cite{holden2010splitting}. Let $v(x,t)$ be the summation of several functions,
  \begin{align}
      v(x,t) = \sum_{j=1}^J v_j(x,t), J \in \mathbb{Z}^+.
  \end{align}
  For a given time step $\Delta t$, we define the iteration as
  \begin{align}\label{eq:iteration_T1}
      x_{k+1} = T_k^{(J)} \circ \dots \circ T_k^{(2)} \circ T_k^{(1)} x_k,
  \end{align}
  where the map $T_k^{(j)}: x \to y$ is
  \begin{align}
      y \equiv T_k^{(j)}(x) = x + \Delta t v_j(x,t_k)
      \quad
      t_k = k \Delta t.
  \end{align}
  
  \begin{lemma}\label{th:split_approach}
  Let all $v_j(x,t), j=1,2,...,J,$ and $v(x,t), (x,t)\in \mathbb{R}^d \times [0,\tau],$ be piecewise smooth (w.r.t. $t$) and $\tilde L$-Lipschitz ($\tilde L>0$); \emph{i.e.,}
  $$\|v_j(x_1,t)-v_j(x_2,t)\| \le \tilde L \|x_1-x_2\|, \forall x_1,x_2 \in \mathbb{R}^d,$$
  $$\|v(x_1,t)-v(x_2,t)\| \le \tilde L \|x_1-x_2\|, \forall x_1,x_2 \in \mathbb{R}^d.$$
  In addition, let the norm of $v_j,v$ and their gradients be bounded in each piece. Then, for any $\tau>0$, $\varepsilon>0$ and $x_0$ in a compact set $\Omega$, there exist a positive integer $n$ and $\Delta t = \tau/n$ such that
  $\|x(k\Delta t) - x_k\| \le \varepsilon$ for all $k\le n$. Particularly, $\|x(\tau) - x_n\| \le \varepsilon$.
  \end{lemma}
  \begin{proof}
  Without loss of generality, we only consider $J=2$ and assume $v_j(x,t)$ are smooth such that $v_j(x,t) \in C^3$. (The general $J$ and the piecewise smooth case can be proven accordingly.) Thus, we have
  \begin{align*}
      x_{k+1} &= T_k^{(2)} (x_k + \Delta t v_1(x_k,t_k))\\
              & = x_k + \Delta t v_1(x_k,t_k) + \Delta t v_2(x_k + \Delta t v_1(x_k, t_k),t_k).
  \end{align*}
  Since $v_2(x,t)$ is smooth, the Taylor expansion w.r.t. $x$ can be employed to obtain
  \begin{align*}
      x_{k+1} 
      &= x_k + \Delta t (v_1(x_k,t_k) + v_2(x_k,t_k)) 
      + \Delta t^2 
      \nabla_x v_2(\eta_k,t_k) \cdot v_1(x_k,t_k) \\
          &=x_k+v(x_k,t_k)\Delta t + 
      \nabla_x v_2(\eta_k,t_k) \cdot v_1(x_k,t_k)
      \Delta t^2,
  \end{align*}
  where $\eta_k = x_k + \xi_k \Delta t v_1(x_k,t_k), \xi_k \in (0,1)$.
  However, the Taylor expansion for $x(t)$ w.r.t. $t$ gives
  \begin{align*}
      x(t_{k+1})
      &=
      x(t_k) + \dot x(t_k)\Delta t+ \frac{1}{2} \ddot x(\lambda_k) \Delta t^2 \\
      &=
      x(t_k) + v(x(t_k), t_k)\Delta t+ 
      \frac{1}{2}
      [
          \nabla_x v \cdot v
      +
          \partial_t v
      ]\lvert_{\lambda_k}
      \Delta t^2,
  \end{align*}
  where $\lambda_k = t_k + \zeta \Delta t, \zeta \in (0,1)$.
  According to the smoothness and boundedness of $v_1$ and $v_2$, there is a positive constant $c$ such that
  \begin{align*}
      \max_j\{1, \|v_j\|,\|\nabla_{(x,t)} v_j\|,
      \|v\|,\|\nabla_{(x,t)} v\|\} \le c.
  \end{align*}
  Defining the error as $e_k:=x_k-x(t_k)$, we have the following estimation:
  \begin{align*}
      \|e_{k+1}\|
      &\le 
      \|e_k\| + 
      \Delta t \|v(x(t_k),t_k)-v(x_k,t_k)\|+ 2c^2 \Delta t^2\\
      &\le (1 + \tilde L\Delta t)\|e_k\| + 2c^2 \Delta t^2.
  \end{align*}
  Employing the inequality $(1+\tilde L\Delta t)^k \le e^{\tilde L k \Delta t } \le e^{\tilde L\tau}$, we have
  \begin{align*}
      \|e_{k}\|
      \le
      (1+\tilde L\Delta t)^k\|e_0\|
      +
      \frac{2c^2 \Delta t^2}{\tilde L\Delta t}[(1+\tilde L\Delta t)^k-1]
      \le 
      2c^2 \Delta t (e^{\tilde L\tau}-1)/ \tilde L .
  \end{align*}
  For any $\varepsilon >0$, $n \ge [\frac{2c^2\tau e^{\tilde L\tau}}{\tilde L \varepsilon} ]$; then, we have $\|x(k\Delta t)-x_k\| \le \varepsilon$, which concludes the proof.
  \end{proof}
  
  Note that the field function $\tilde v(x,t)$ in (\ref{eq:tanh_NN}) or (\ref{eq:specified_v})
  satisfies the condition of the above lemma. As an application, the iteration in (\ref{eq:main_iteration_Tk}) can be used to discretize the ODE in (\ref{eq:ODE_neural}) with the field function (\ref{eq:specified_v}).
  
  \subsection{Approximate each split step by an FNN}
  Our main results rely on the following construction for the iteration $T_k^{(i,j)}$ with the specified $\tilde v$ in (\ref{eq:tanh_NN}) or (\ref{eq:specified_v}), which is a tanh network. Since each $T_k^{(i,j)}$ has the same structure (over a permutation), so we only need to consider the case of $T_k^{(N,d)}$, which is simply denoted as $T:x\to y$,
  \begin{align}\label{eq:iteration_specified_T}
      T: \left\{
      \begin{aligned} 
      & y^{(i)} = x^{(i)} , i = 1,\cdots, d-1,  \\
      & y^{(d)} = x^{(d)} + a \Delta t  \tanh(w \cdot x + \beta),
          \end{aligned}
      \right.
  \end{align}
  where $a=a_{N,d}$ is the $(N,d)$-th element of $A(t_k)$, $w$ is the $N$-th row of $W(t_k)$, and $\beta$ is the $N$-th element of $b(t_k)$.
  
  \begin{theorem}\label{th:NN_for_tahn_flow}
  If $\mathcal{K}$ is a compact set and $\Delta t$ in map $T:x\to y$ is sufficiently small, then for any $\varepsilon>0$, there is a leaky-ReLU network $f_L(x) \in \mathcal{N}(L)$ such that
  \begin{align}\label{eq:in_th_T_f_L}
          \|T(x) - f_L(x)\| \le \varepsilon, \forall x \in \mathcal{K}.
      \end{align}
  \end{theorem}
  
  \begin{proof}
  Here, $C := \max\{\|A\|_\infty,\|W\|_\infty\} $, $w = (w_1,...,w_d)$, $d \geq 2$, and $\Delta t < \frac{1}{{C}^2}$; then, $T$ is rewritten as
  \begin{align}\label{eq:iteration_2dim_T}
      T: \left\{
      \begin{aligned} 
      & y^{(i)} = x^{(i)} ,  i = 1,\cdots, d-1,\\
      & y^{(d)} = x^{(d)} + a \Delta t  \tanh(w_{1} x^{(1)}+ \cdots +w_{d} x^{(d)} + \beta).
          \end{aligned}
      \right.
  \end{align}
  We show how to construct a leaky-ReLU network $f_L(x)$ to approximate $T(x)$. Note that the case of $w_1=...=w_{d-1}= 0$ is trivial according to Lemma \ref{th:main_1d}. The reason is that $x^{(d)} + a \Delta t  \tanh(w_{d} x^{(d)} + \beta)$ is monotonic for $x^{(d)}$; then, we can approximate each element of $T$ independently.
  
  Next, we consider the case of $\sum_{i=1}^{d-1} w^2_{i} \neq 0 $. Without loss of generality, we assume $w_1 \neq 0$. We show that map $T$ can be represented by the following composition:
  \begin{align}
      T(x) \equiv F_6 \circ \cdots \circ F_0(x),
  \end{align}
  where each mapping step is as follows:
  \begin{align*} 
      \left(
          \begin{matrix} x^{(1)} \\ x^{(2:d-1)} \\ x^{(d)} 
          \end{matrix}\right)
      &\underrightarrow{F_0}
      \left(
          \begin{matrix} w_{1}x^{(1)}+\cdots+w_{d}x^{(d)}+\beta \\ x^{(2:d-1)}\\ x^{(d)} 
          \end{matrix}\right)
      \equiv
      \left(
          \begin{matrix} \nu \\ x^{(2:d-1)}\\ x^{(d)} 
          \end{matrix}\right)    \\
      &\underrightarrow{F_1}
      \left(
          \begin{matrix} \tanh{(\nu)} \\ x^{(2:d-1)}\\ x^{(d)}
          \end{matrix}\right)
      \underrightarrow{F_2}
      \left(
          \begin{matrix} \tanh{(\nu)} \\ x^{(2:d-1)}\\ x^{(d)}+ a \Delta t \tanh{(\nu)} 
          \end{matrix}\right)
      \underrightarrow{F_3}
      \left(
          \begin{matrix} \nu \\ x^{(2:d-1)}\\ x^{(d)} + a \Delta t \tanh{(\nu)} 
          \end{matrix}\right)\\
      &\underrightarrow{F_4}
      \left(
          \begin{matrix} \nu+w_{d} a \Delta t \tanh{(\nu)} \\ x^{(2:d-1)}\\ x^{(d)}+ a \Delta t \tanh{(\nu)} 
          \end{matrix}\right)
      \underrightarrow{F_5}
      \left(
          \begin{matrix} x^{(1)}+\frac{\beta}{w_{1}} \\ x^{(2:d-1)}\\ x^{(d)}+ a \Delta t \tanh{(\nu)} 
          \end{matrix}\right)\\
      &\underrightarrow{F_6}
      \left(
          \begin{matrix} x^{(1)} \\ x^{(2:d-1)}\\ x^{(d)}+ a \Delta t \tanh{(\nu)} 
          \end{matrix}\right).
  \end{align*}
  Here, $\nu:=w_{1}x^{(1)}+\cdots+w_{d}x^{(d)}+\beta$ and $x^{(2:d-1)}$ represent the elements $x^{(2)},...,x^{(d-1)}$.
  We clarify that each component $F_i,i=0,\cdots,6,$ can be represented or approximated by leaky-ReLU networks $\tilde F_i := f_{i,L_i}(x)$.
  \begin{itemize}
  \item[$F_0,$]$F_2, F_5, F_6.$ These steps can be achieved by simple linear transformations, which are denoted by $\tilde F_0, \tilde F_2, \tilde F_5, \tilde F_6$.
  
  \item[$F_1.$] Since $\tanh(\nu)$ is monotonic, according to Lemma \ref{th:main_1d}, for any $ \varepsilon_1 >0$, we can find a leaky-ReLU network $f_{1,L_1}$ such that $\|f_{1,L_1}(z)-F_1(z)\| < \varepsilon_1$ for all $z$ in any given compact domain.

  \item[$F_3.$] Since the function $\text{arctanh}(\mu)$ is monotonic and $\mu=\tanh(\nu)$ is in an open interval $(-1,1)$, we can obtain the same conclusion as $F_1$; \emph{i.e.,} for any $ \varepsilon_3 >0$, we can find a leaky-ReLU network $f_{3,L_3}$ such that $\|f_{3,L_3}(z)-F_3(z)\| < \varepsilon_3$ for all $z$ in any given compact domain with $z_1 \in (-1,1)$.
  
  \item[$F_4.$] Since $ w_d a\Delta t < \frac{\Delta t}{{C}^2} < 1$, $\nu+w_{d} a \Delta t  \tanh(\nu)$ is monotonic for $\nu$ and $\nu$ is in a compact domain; for any $ \varepsilon_4 >0$, we can find a leaky-ReLU network $f_{4,L_4}$ such that $\|f_{4,L_4}(z)-F_4(z)\| < \varepsilon_4$ for all $z$ in any given compact domain.
  
  \end{itemize}
  By combining the $\tilde F_i$ networks above and adopting Lemma \ref{th:composition_approximation} below, $f_L := \tilde F_6 \circ \cdots \circ \tilde F_0$ is a leaky-ReLU network, and we can specify the values of $\varepsilon_i$ such that $f_L$ satisfies the condition in (\ref{eq:in_th_T_f_L}). The proof of the theorem is now complete.
  \end{proof}
  
  \begin{lemma}\label{th:composition_approximation}
  Let map $T = F_n \circ ... \circ F_1$ be a composition of $n$ continuous functions $F_i$ defined on an open domain $D_i$, and let $\mathcal{F}$ be a continuous function class that can uniformly approximate each $F_i$ on any compact domain $\mathcal{K}_i \subset D_i$. Then, for any compact domain $\mathcal{K} \subset D_1$ and $\varepsilon >0$, there are $n$ functions $\tilde F_1, ..., \tilde F_n$ in $\mathcal{F}$ such that
  \begin{align}
          \|T(x) - \tilde F_n \circ ... \circ \tilde F_1 (x)\|
          \le \varepsilon,
          \quad
          \forall x \in \mathcal{K}.
      \end{align}
  \end{lemma}
  \begin{proof}
  It is enough to prove the case of $n=2$. (The case of $n>2$ can be proven by the method of induction, as $T$ can be expressed as the composition of two functions, $T = F_n \circ T_{n-1}$, with $T_{n-1} = F_{n-1} \circ ... \circ F_1$.) According to the definition, we have $F_1(D_1) \subset D_2$. Since $D_2$ is open and $F_1(\mathcal{K})$ is compact, we can choose a compact set $\mathcal{K}_2 \subset D_2$ such that $\mathcal{K}_2 \supset \{F_1(x) + \delta_0  y: x\in \mathcal{K}, \|y\|<1 \} $ for some $\delta_0>0$ that is sufficiently small.
  
  According to the continuity of $F_2$, there is a $\delta \in (0,\delta_0)$ such that
  \begin{align*}
          \|F_2(y) - F_2(y')\| &\le \varepsilon/2, \forall y,y' \in \mathcal{K}_2,
      \end{align*}
  provided $\|y-y'\| \le \delta$.
  The approximation property of $\mathcal{F}$ allows us to choose $\tilde F_1, \tilde F_2 \in \mathcal{F}$ such that
  \begin{align*}
          \|\tilde F_1(x) - F_1(x)\| &\le \delta < \delta_0, \quad \forall x \in \mathcal{K}, \\
          \|\tilde F_2(y) - F_2(y)\| &\le \varepsilon/2, \quad \forall y \in \mathcal{K}_2.
      \end{align*}
  As a consequence, for any $x \in \mathcal{K}$, we have $F_1(x), \tilde F_1(x) \in \mathcal{K}_2$ and
  \begin{align*}
          \|F_2 \circ F_1(x) - \tilde F_2 \circ \tilde F_1(x)\|
          &\le
          \|F_2 \circ F_1(x) - F_2 \circ \tilde F_1(x)\|
          +
          \|F_2 \circ \tilde F_1(x) - \tilde F_2 \circ \tilde F_1(x)\|\\
          &\le
          \varepsilon/2 + \varepsilon/2 = \varepsilon.
      \end{align*}
  \end{proof}
  
  This lemma indicates that we can approximate $T$ by approximating each composition component of $T$.
  
  \subsection{Proof of {\changes Theorem \ref{th:main}} }
  
  Now, we can complete the proof of our main theorem.
  
  \begin{proof}
  We divide our proof into three steps.
  
  First, according to Lemma \ref{th:NODE_approx_DS}, there exists a flow map $\tilde{\phi}^\tau(x)$ of a neural ODE (\ref{eq:ODE_neural}) with field function $\tilde v$ in (\ref{eq:specified_v}) or (\ref{eq:tanh_NN}) and a tanh network with $N$ neurons such that
  \begin{align}
      \|\tilde{\phi}^\tau(x)-\phi^\tau(x)\| \le \frac{\varepsilon}{3}, \forall x \in \Omega.
  \end{align}
  
  Next, we use the splitting method in (\ref{eq:iteration_T1}) with time step $\Delta t = \tau/n, n\in \mathbb{Z}^+$ to approximate $\tilde{\phi}^\tau(x)$. Recalling the interaction in (\ref{eq:main_iteration_Tk}), the numerical solution for $\tilde \phi^\tau(x_0)$ is given by
  \begin{align*}
      x_n &= \Phi(x_0) := T_n \circ \cdots \circ T_1 (x_0) \\
      &=
      T_n^{(N,d)} \circ \dots \circ T_n^{(1,2)} \circ T_n^{(1,1)} 
      \circ
      \cdots \circ
      T_1^{(N,d)} \circ \cdots \circ T_1^{(1,2)} \circ T_1^{(1,1)} 
      (x_0),
  \end{align*}
  where $T^{{i,j}}_k$ is defined by (\ref{eq:map_T}). Note that the tanh network in (\ref{eq:specified_v}) or (\ref{eq:tanh_NN}) satisfies the condition in Lemma \ref{th:split_approach}. Therefore, we can find a sufficiently small $\Delta t$ (sufficiently large $n$) such that
  \begin{align}
      \|\Phi(x)-\tilde{\phi}^\tau(x)\| \le \frac{\varepsilon}{3},
      \forall x \in \Omega.
  \end{align}
  The maps $T^{{i,j}}_k$ in $\Phi$ are smooth and Lipschitz on $\mathbb{R}^d$. Theorem \ref{th:NN_for_tahn_flow} shows that each $T_k^{(i,j)}$ (constrained on any compact domain) can be uniformly approximated by leaky-ReLU networks. As a consequence, adopting Lemma \ref{th:composition_approximation} again, there is a leaky-ReLU network $f_L \in \mathcal{N}_d$ such that
  \begin{equation}
      \|\Phi(x)-f_L(x)\| \le \frac{\varepsilon}{3},
      \forall x \in \Omega.
  \end{equation}
  
  Finally, for any $x \in \Omega$, we have
  \begin{align*}
      \|f_L(x)-{\phi}^\tau(x)\| &
      \le 
      \|f_L(x)-\Phi(x)\|
      +\|\Phi(x)-\tilde{\phi}^\tau(x)\|
      +\|\tilde{\phi}^\tau(x)-\phi^{\tau}(x)\|\\ 
      & \le \frac{\varepsilon}{3}+\frac{\varepsilon}{3}+\frac{\varepsilon}{3} 
      = \varepsilon.
  \end{align*}
  The proof of the theorem is now complete.
  \end{proof}

  \section{Discussion}
  \label{sec:discussion}
  
  Until now, we only considered networks with leaky ReLU activation functions. However, our results can be extended to a broad class of activation functions. Here, we discuss the requirements for the activation functions.
  
  First, we emphasize that our results do not hold for ReLU networks. This argument is based on the following properties, which imply the poor approximation power of ReLU networks with widths of one.
  
  \begin{proposition}\label{th:relu_one}{\changess}
  ReLU neural networks with a width of one are PL functions with at most three pieces.
  \end{proposition}
  \begin{proof}
  Let $\sigma_0$ be a ReLU activation function, and consider the ReLU network $f_L(x)$ with width one and depth $L$; that is,
  $f_0(x)=w_0 x+b_0, f_{i}(x)=w_{i}\sigma_0(f_{i-1}(x))+b_{i}$, $i=1,\cdots,L$. Similar to the leaky-ReLU case, we can see that all $f_i(x)$ are monotonic. If there is a $w_i=0$, then $f_L(x)$ is a constant function that is trivial. For the other case, we can assume $w_0=1$ and $w_i \in \{-1,1\}$ due to the positive homogeneity of ReLU.
  
  If $w_0=w_1=\cdots=w_{L}=1$, then $f_L(x)=\sigma_0(x+\tilde{a}_L)+\tilde{b}_L$ for some $\tilde{a}_L$ and $\tilde{b}_L$, which has only one breakpoint.
  
  If $w_0=\cdots=w_{n-2}=1,w_{n-1}=-1$ for $n\le L$, then $g_n(x):=\sigma_0(f_{n-1}(x))$ has the following form:
  $$g_n(x):=\sigma_0(f_{n-1}(x))=\sigma_0(-\sigma_0(x+\tilde{a}_{n-1})-\tilde{b}_{n-1}),$$
  for some $\tilde{a}_{n-1}$ and $\tilde{b}_{n-1}$. It is obvious that $g_n(x)$ and $f_n(x) = w_n g_n(x) + b_n$ are bounded as $0\le g_n(x)\le \max \{0,-\tilde{b}_{n-1}\}$. In addition, $f_n(x)$ is a bounded PL function with at most two breakpoints. As a consequence, adding more ReLU layers does not increase the number of breakpoints. Therefore, $f_L(x)$ has at most three pieces.
  \end{proof}
  
  Next, we note that our results hold for activation functions that can approximate leaky ReLU functions. Below are some examples (they are assumed to be continuous; see also the examples shown in Figure \ref{fig:activation}).
  \begin{itemize}
  \item[1)] PL/smooth functions that have at least one strictly increasing or decreasing breakpoint. In other words, leaky-ReLU is (almost) one of these functions. For example, $\sigma(x) = \max(x^3-1,x)$.
  \item[2)] Unbounded functions that have two asymptotic lines with different nonzero slopes, such as $\sigma(x) = x + \ln(1+e^x)$.
  \end{itemize}
  Of course, it is not necessary for these functions to be monotonic. If we only want the approximation results (without the homeomorphism-preserving property), nonmonotonic activation functions, such as $\sigma(x) = \max(x^2-1,x)$, can satisfy the main results.
  Our proofs show that it is essential that $\sigma$-networks with widths of one can universally approximate monotone functions on any bounded interval. Our construction shows that leaky ReLU satisfies this condition and that ReLU does not. However, we do not know the answer for more general activation functions (except the abovementioned examples), such as $\sigma(x) = \tanh(x), \ln(1+e^x), \sin(x), x^2$ and $\text{ReLU}^2(x)$.

  \begin{figure}[htp!]
        \centering
        \includegraphics[width=9cm]{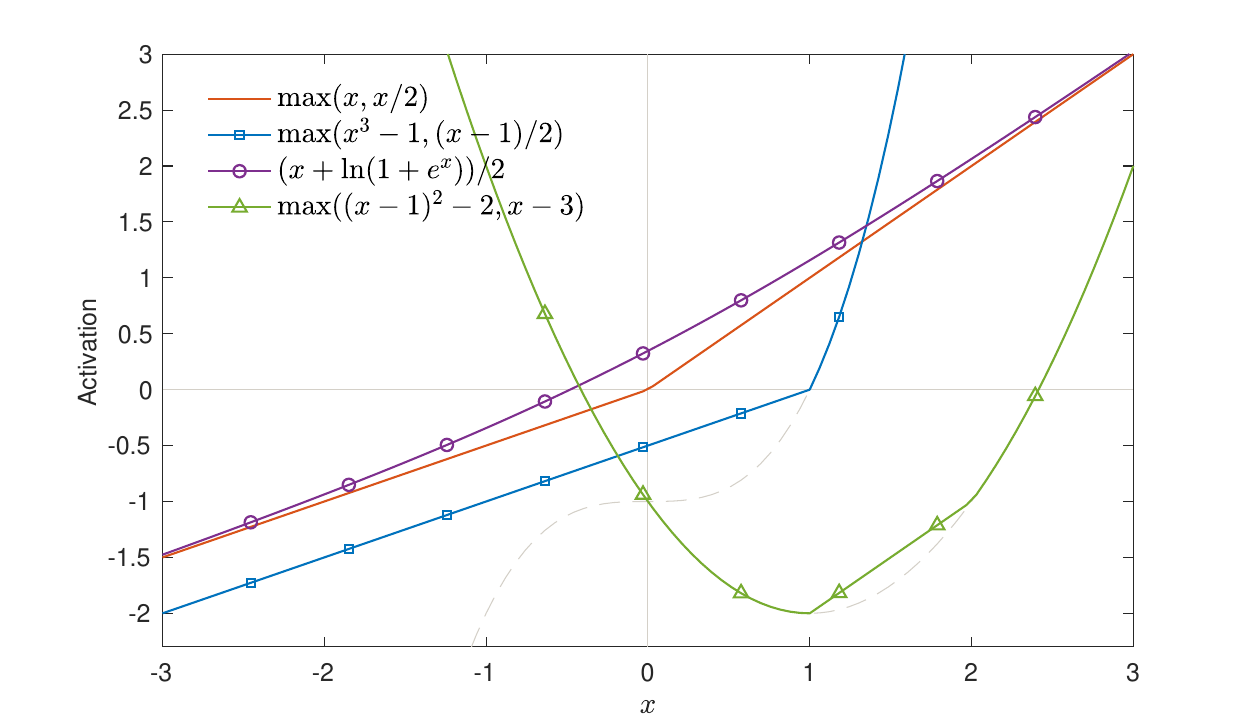}
    \caption{Four examples of activation functions that satisfy our main results. }
        \label{fig:activation}
  \end{figure}

  {\changes

\begin{figure}[htp!]
        \centering
        \includegraphics[width=9cm]{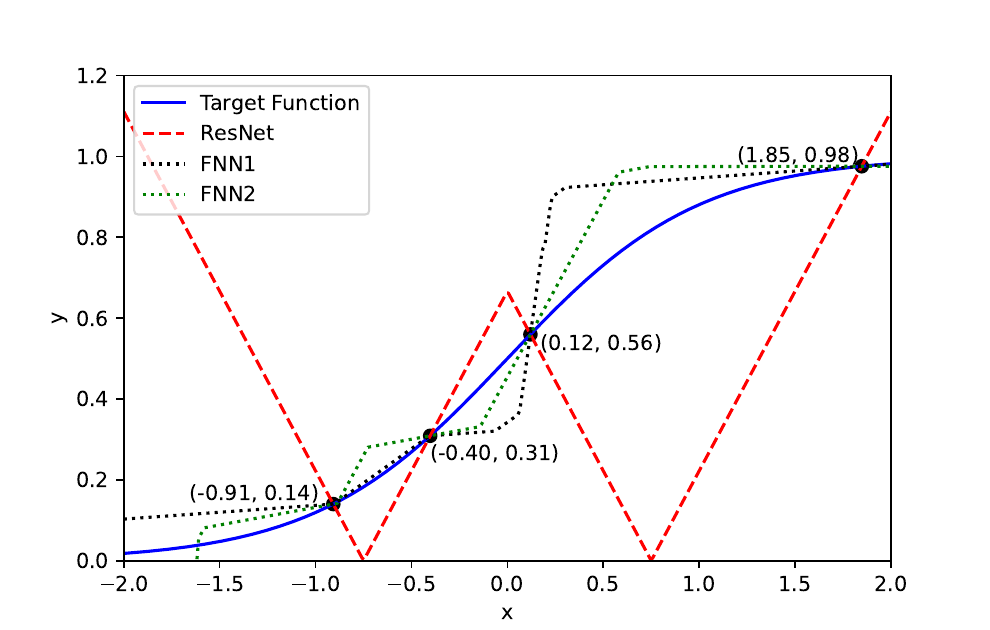}
    \caption{A 1D experiment to show the difference between FNNs and ResNets. }
        \label{fig:experiments}
  \end{figure}

  \begin{figure}[htp!]
        \centering
        \includegraphics[width=12cm]{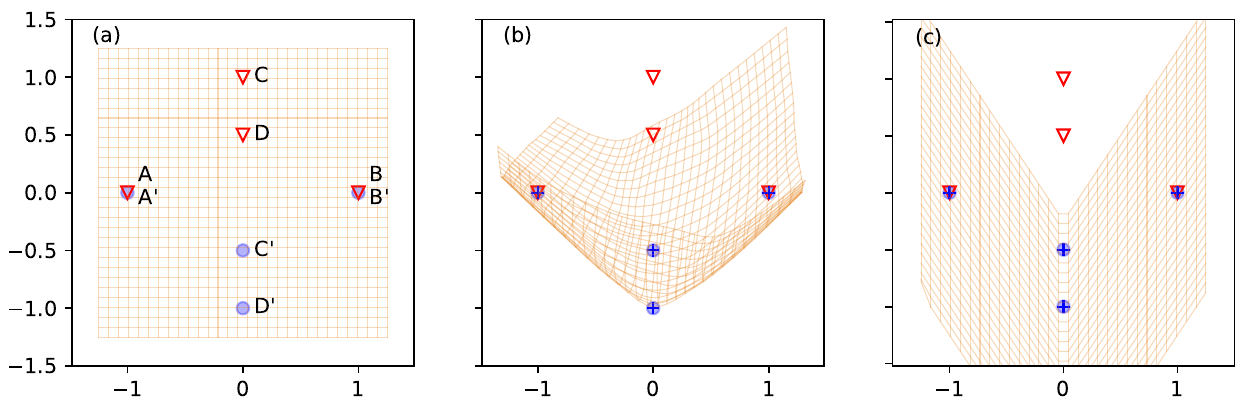}
    \caption{{\changess A 2D experiment demonstrating the difference between wide and narrow feedforward neural networks. (a) A toy task to construct a homeomorphism that maps points A, B, C, and D to points A', B', C', and D', respectively. (b) Result from a wide FNN (one hidden layer, width = 100). (c) Result from a narrow FNN (width = 2, depth = 5).}}
        \label{fig:experiments2d}
  \end{figure}

  Then, we want to note the difference between FNNs and ResNets. We use a simple numerical experiment to emphasize the homeomorphic property. Consider a target function $u(x) = \tfrac{1}{2}(1+\tanh (x)), x\in [-2,2],$ with four given training points shown in Figure \ref{fig:experiments}. When we use leaky-ReLU FNNs with width one to fit these points, the networks are always monotonic. However, the ResNet with one hidden neuron is not always monotonic. In Figure \ref{fig:experiments}, we show a nonmonotonic ResNet that has two blocks with expression $y = 8 R_2 \circ R_1(x)$ where $R_1(x) = x - \tfrac{4}{3} \sigma_{0.5}(x)$ and $R_2(x) = x + \tfrac{4}{3} \sigma_{0.5}(-x-0.25) + 0.25$. This indicates that ResNets do not preserve the homeomorphic property.

{\changess

Next, we address the difference between wide and narrow FNNs. We continue to emphasize the homeomorphic property by using a two-dimensional toy task, as shown in Figure~\ref{fig:experiments2d}(a). The task is to construct a homeomorphism that maps points A, B, C, and D to points A', B', C', and D', respectively. A mesh grid is provided to illustrate the effect of the constructed map. Figures~\ref{fig:experiments2d}(b) and \ref{fig:experiments2d}(c) show the results of a wide ReLU FNN (one hidden layer, width = 100) and a narrow leaky-ReLU FNN (width = 2, depth = 5), respectively. The wide FNN is randomly initialized, while the narrow FNN is initialized according to the structure studied in this paper. Both FNNs are trained using the Adam optimizer \cite{kingma2014adam} with a learning rate of 0.001. It is clear that the learned wide FNN does not always satisfy the requirements of a diffeomorphism.

From a training perspective, it is important to highlight the difference between the FNN constructed in this paper and an FNN trained from scratch. The FNN we construct preserves the homeomorphic property of flow maps because we use a nonsingular weight matrix in each layer. However, nonsingularity is not guaranteed during the training process. One way to address this issue is to check the singularity of the weight matrices and perturb any singular matrices to make them nonsingular during training. In fact, if the weight matrix becomes singular, the leaky-ReLU network loses both the homeomorphic property and its approximation power. Therefore, monitoring and addressing singularity could play a crucial role in the training process.
}

  Finally, we want to note that the results in this paper only focus on the approximation ability of FNNs. It would be interesting to analyze the approximation rate. For one-dimensional cases, we estimate the required depth in Remark~\ref{remark:L_bound_1d}. For higher dimensional cases, such estimation is challenging. In Lemma~\ref{th:split_approach}, we estimate the number of splitting steps $n = O({J \tau c^2 e^{\tilde L \tau}}/{(\tilde L\varepsilon)})$ with $J=Nd$ for the splitting in (\ref{eq:splitting_of_tilde_v}). However, this is far from giving the required depth $L$ because it is also affected by the approximation in Step 1 and Step 3 provided in Section~\ref{sec:main}. Systemly studying this topic is beyond the scope of this paper. We leave it as future work.

  }

  {\changes
  \section{Conclusion}
  \label{sec:conclusion}

  This paper employs {\changess numerical approximation techniques} to construct a cross-disciplinary theoretical framework. 
  By devising an appropriate operator splitting algorithm, a connection between the flow map of dynamical systems and the feedforward neural network is established, in which the width of the neural network equals the dimensionality of its inputs and outputs. It is expected that this linkage will foster research in multiple domains, including the investigation into the minimum width required for a feedforward neural network to possess universal approximation properties, as well as the exploration of the universal approximation properties of flow maps.

  ~\\
  }

  \textbf{Funding} G. Ji is partially supported by the National Natural Science Foundation of China (Grant No. 11871105), and Y. Cai is partially supported by the National Natural Science Foundation of China (Grant No. 12201053).
  
  \textbf{Data Availability:} Data sharing is not applicable to this article as no datasets were generated or analyzed during
  the current study.
  
  

  %
  %

  \bibliographystyle{spmpsci}      
\bibliography{refs.bib}   
  
  \end{document}